%% file: main.tex
\documentclass{article}

\usepackage[preprint]{corl_2026} 
\usepackage[utf8]{inputenc} 
\usepackage[T1]{fontenc}    
\usepackage{hyperref}       
\usepackage{url}            
\usepackage{booktabs}       
\usepackage{amsfonts}       
\usepackage{nicefrac}       
\usepackage{float}          
\usepackage{microtype}      
\usepackage{xcolor}         
\usepackage{amsmath,amssymb,amsfonts}
\usepackage{longtable}
\usepackage{booktabs}
\usepackage{multirow}

\usepackage{bbm}
\usepackage{graphicx}
\usepackage{textcomp}
\usepackage{tabularx}
\usepackage{adjustbox}
\usepackage{array}
\usepackage{makecell}
\usepackage{xcolor}
\usepackage[table]{xcolor}
\usepackage{algorithm}
\usepackage{cleveref}
\usepackage{graphicx}
\usepackage{multirow}
\usepackage{enumitem}
\usepackage{algpseudocode}
\usepackage{mathtools}
\usepackage[english]{babel}
\usepackage{amsthm}
\usepackage[algo2e,ruled,linesnumbered]{algorithm2e}
\usepackage{multirow}
\usepackage{lipsum}
\usepackage{subcaption}
\usepackage{graphicx}
\usepackage{amsthm}
\newtheorem{theorem}{Theorem}[section]

\newtheorem{lemma}[theorem]{Lemma}

\theoremstyle{definition}

\title{Safe-RULE: Safe Reinforcement UnLEarning}

\author{
Shixiong Jiang \quad Taozheng Zhu \quad Fanxin Kong\\
Department of Computer Science and Engineering, University of Notre Dame\\
\texttt{sjiang5@nd.edu} \quad
\texttt{tzhu2@nd.edu} \quad
\texttt{fkong@nd.edu}
}

\makeatletter
\providecommand{\@noticestring}{}
\makeatother
\begin{document}
\maketitle


\begin{abstract}
Offline safe reinforcement learning (Safe RL) enables policy learning without online interactions, making it suitable for safety-critical systems such as robotics systems. However, its reliance on static datasets exposes offline Safe RL to data poisoning attacks, where adversaries inject malicious samples that compromise safety and induce unsafe policy behavior. In this work, we propose a new learning paradigm, named safe reinforcement unlearning (Safe-RULE), used as   a defense framework to remove the influence of poisoned data without retraining from scratch or requiring access to the original training environment. We further extend reinforcement unlearning to offline Safe RL by explicitly accounting for both task performance and safety constraints during the unlearning process. Experiments across benchmark Safe RL tasks demonstrate that our approach effectively enhances safety performance against data poisoning attacks.
\end{abstract}

\input{Sections/01_Introduction}

\input{Sections/02_RelatedWork}

\input{Sections/03_Framework}

\input{Sections/04_}

\input{Sections/05_}
\input{Sections/06_Evaluation}

\input{Sections/07_conclusion}


\clearpage


\bibliography{ref}  

\input{Sections/08_Appendix}

\end{document}

%% file: Sections/01_Introduction.tex
\section{Introduction}

Offline safe reinforcement learning (safe RL) has been widely explored in robotics \cite{sinha2022s4rl,li2024guided}, autonomous driving \cite{shi2021offline,fang2022offline}, and industrial manipulation \cite{gu2024review}. In the typical offline safe RL pipeline, a dataset is first collected from one or more behavior policies. Then, a safe offline RL algorithm is trained on this fixed dataset (no further interaction with the environment during training). Finally, the learned safe policy is evaluated in the environment.

This setup avoids unsafe online exploration, but it also incurs  vulnerability. Several works have shown that the robustness of an offline-trained “safe” policy can be seriously reduced if part of the dataset is malicious \cite{gong2024baffle}. The attack surface is large: an adversary can inject poisoned trajectories~\cite{gong2022baffle}, collect data from a compromised environment \cite{ye2023reinforcement}, or insert backdoor-style triggers that bias behavior \cite{kiourti2020trojdrl,jiang2024backdoor}. These poisoned samples can modify the observations, push the policy toward unsafe actions, or tamper with the reward/cost labels. As a result, the final policy may violate safety constraints at deployment, even though it was trained with a nominally safe RL algorithm.

However, defense methods for repairing poisoned offline safe RL policies remain underexplored: how do we repair a safe policy that was already trained on poisoned data? A straightforward option is to retrain the whole policy from scratch using only clean data, but it is computationally expensive~\cite{ye2023reinforcement,gong2024trajdeleter}. Another common option is to fine-tune the current policy on only the clean subset, but it does not guarantee the removal of the poisoned behavior, since the poisoned behavior has already been encoded in the policy. Overall, there is a need to develop a new method to counter the poison attacks effectively and efficiently.

To address this problem, we propose \textbf{Safe Reinforcement Unlearning (Safe-RULE)}, which enables a safe policy to unlearn poisoned samples without retraining from scratch or re-accessing the original training environment, while maintaining comparable performance on the remaining clean data. In this paper, we aim to answer the following questions: 
\emph{(1) How can unlearning be applied to offline safe RL?} 
\emph{(2) How can safe reinforcement unlearning balance task performance and safety cost?}
To the best of our knowledge, we are the first to explore reinforcement unlearning for safe RL policies. Our contributions are summarized as follows:
\begin{enumerate}[topsep=0pt,itemsep=0pt,leftmargin=*]
    \item We formulate the problem of safe reinforcement unlearning and propose Safe-RULE, an unlearning framework for offline safe RL policies trained on poisoned datasets.
    \item We conduct extensive experiments on Safety Gym benchmarks to demonstrate the effectiveness and efficiency of Safe-RULE against various poisoning attacks.
\end{enumerate}

%% file: Sections/02_RelatedWork.tex
\vspace{-0.1in}
\section{Related Work}
\textbf{Machine unlearning.}
Machine unlearning aims to remove the influence of specific training data from neural networks. Existing machine unlearning applications mainly target removing harmful knowledge \cite{wu2023depn,farrell2024applying,deeb2024unlearning}, mitigating jailbreaks \cite{sheshadri2024latent,isonuma2024unlearning}, correcting value alignment \cite{ouyang2022training,vidal2024verifying}, and satisfying privacy \cite{sula2024silver}. Current unlearning approaches consist of: (i) gradient ascent \cite{gu2024second, jang2023knowledge}, which maximizes the loss on the samples to be forgotten; (ii) task-vector methods \cite{ilharco2022editing,zhang2023composing}, which identify an unlearning direction directly in weight space; and (iii) representation misdirection \cite{huu2024effects, rosati2024representation,li2024wmdp}, which drives intermediate representations for the forgotten data toward noise or irrelevant features.

Despite the broad use of machine unlearning, machine unlearning for reinforcement learning has recently emerged. \cite{ye2023reinforcement} first propose reinforcement unlearning for online RL, in which they propose two methods for unlearning to forget a poisoned environment while keeping substantial performance for other environments. \cite{gong2024trajdeleter} further proposes Trajdeleter, which focuses on offline RL and unlearning poisoned trajectories instead of a whole environment. However, neither work studies machine unlearning for safe RL policies. 

\textbf{Poison attacks for RL.} Poison attacks refer to the manipulation of the training dataset, resulting in the training RL having compromised performance \cite{behzadan2017vulnerability, huang2019deceptive,rakhsha2020policy, liu2021provably,lin2017tactics}. The poison attacks for RL can be implemented through perturbing the observation, perturbing the action, perturbing the reward function, or a combination of them to compromise the reward of the policy. Instead of degrading the reward, \cite{liu2022robustness} is the first to propose adversarial attacks that target safe RL. \cite{liu2022robustness} points out that the adversarial attacks can perturb the observation of safe RL agents by either maximizing the reward or the cost value so that the trained policy gets an increased cumulative cost, which leads to unsafe behavior. Following this work, several works \cite{jiang2024backdoor,guo2025pnact} have developed various adversarial attack methods to explore the vulnerability of safe RL algorithms. However, the above works lack the discussion of how general defense methods can be developed to prevent various attacks, especially for offline safe RL. 

%% file: Sections/03_Framework.tex
\vspace{-0.1in}
\section{Preliminary}

\subsection{Offline Safe RL}

We model safe reinforcement learning as a Constrained Markov Decision Process (CMDP)
\[
\mathcal{M}
=
(\mathcal{S},\mathcal{A},\mathcal{P},r,c,\gamma,\mu_0),
\]
where \(\mathcal{S}\) and \(\mathcal{A}\) are the state and action spaces, \(\mathcal{P}(s'\mid s,a)\) is the transition kernel, \(r:\mathcal{S}\times\mathcal{A}\to\mathbb{R}\) is the reward function, and \(c:\mathcal{S}\times\mathcal{A}\to[0,c_{\max}]\) is a non-negative cost function bounded by \(c_{\max}\). Here, \(\gamma\in[0,1)\) is the discount factor, and \(\mu_0\) is the initial-state distribution.

For a stationary policy \(\pi\), the discounted reward and cost returns are defined as
\[
J_r(\pi)
=
\mathbb{E}_{\tau\sim\pi}
\left[
\sum_{t=0}^{\infty}\gamma^t r_t
\right],
\qquad
J_c(\pi)
=
\mathbb{E}_{\tau\sim\pi}
\left[
\sum_{t=0}^{\infty}\gamma^t c_t
\right],
\]
where \(\tau\sim\pi\) denotes trajectories generated under \(\mu_0\), \(\mathcal{P}\), and \(\pi\). The objective of safe RL is to learn a policy that maximizes the expected reward while satisfying a safety constraint on the expected cost:
\begin{equation}
\label{eq:cmdp}
\pi^\star
\in
\arg\max_{\pi} J_r(\pi)
\quad
\text{s.t.}
\quad
J_c(\pi)\le \kappa,
\end{equation}
where \(\kappa\) is the cost threshold. We refer to a policy as safe if it satisfies the constraint in~\eqref{eq:cmdp}.

In the \emph{offline} setting, the learner has no access to the environment and must solve~\eqref{eq:cmdp} using only a fixed dataset
\[
\mathcal{D}
=
\{\tau_i\}_{i=1}^{N},
\qquad
\tau_i
=
(s_0,a_0,r_0,c_0,\ldots,s_T,a_T,r_T,c_T),
\]
collected by an unknown behavior policy.

\vspace{-0.1in}
\subsection{Problem Formulation}
\label{sec:problem}

\textbf{Threat Model.}
Let $\tilde{\mathcal{D}}=\mathcal{D}_k\cup\mathcal{D}_f$ be a corrupted offline safe RL dataset. We assume an adversary has injected a subset $\mathcal{D}_f$ through data poisoning, leaving the remaining data $\mathcal{D}_k=\tilde{\mathcal{D}}\setminus\mathcal{D}_f$ clean. We refer to $\mathcal{D}_f$ as the \emph{forget (poison) dataset} and $\mathcal{D}_k$ as the \emph{keep (clean) dataset}. Therefore, the offline safe RL training on the corrupted dataset $\tilde{\mathcal{D}}$ produces a policy $\tilde{\pi}$ such that $J_c(\tilde{\pi}) > \kappa,$ indicating a violation of the safety constraint. This typically occurs because the policy underestimates the true cost on the forget dataset, which allows the corrupted policy to select unsafe actions.

We assume that the partition $(\mathcal{D}_k,\mathcal{D}_f)$ is known at unlearning time, which is standard in the unlearning literature where the forget set is identified by an upstream detection or auditing mechanism.

\textbf{Unlearning Goal.}
A straightforward solution against the corrupted policy is to discard the corrupted policy and retrain from scratch using the clean dataset $\mathcal{D}_k$. However, this approach is computationally expensive. Instead, we aim to \emph{unlearn} the forget set: starting from $\tilde{\pi}$, we seek to obtain a new policy $\pi$ that suppresses the influence of $\mathcal{D}_f$ while preserving performance on $\mathcal{D}_k$.

Specifically, starting from the corrupted policy $\tilde{\pi}$, we aim to find a new policy $\pi$ by solving the following constrained optimization problem:

\begin{equation}
\label{eq:unlearn}
\begin{aligned}
\max_{\pi} \quad & \mathbb{E}_{(s,a)\sim\mathcal{D}_k}\!\left[Q^\pi_r(s,a)\right] \\
\text{s.t.} \quad & \mathbb{E}_{(s,a)\sim\mathcal{D}_k}\!\left[Q^\pi_c(s,a)\right] \leq \kappa, \\
& \mathbb{E}_{(s,a)\sim\mathcal{D}_f}\!\left[Q^\pi_c(s,a)\right] > \kappa,
\end{aligned}
\end{equation}
\noindent where $Q^\pi_r(s,a)$ and $Q^\pi_c(s,a)$ are the state-action value functions for reward and cost under policy $\pi$, respectively. The objective maximizes the expected return on the clean data $\mathcal{D}_k$, subject to two constraints: (1) the policy remains safe on the clean distribution, and (2) the influence of the forget set $\mathcal{D}_f$ is suppressed by ensuring that the cost values on poisoned transitions exceed the safety threshold $\kappa$, effectively rendering those transitions undesirable. Together, these conditions capture the dual requirement of \emph{safety constraint} on $\mathcal{D}_k$ and \emph{forget set suppression} on $\mathcal{D}_f$.



%% file: Sections/04_.tex
\vspace{-0.1in}
\section{Safe-RULE}
In this section, we introduce our safe reinforcement unlearning method \textbf{Safe-RULE} to mitigate poisoned data in the corrupted safe RL policy.

To illustrate the intuition behind our unlearning algorithm design, we first revisit state-of-the-art offline safe RL algorithms. We consider a general actor-critic framework with a reward critic $Q^r_\phi$, a cost critic $Q^c_\psi$, and an actor $\pi_{\theta}$, where $\phi, \psi,\theta$ denote their parameters, respectively. Offline safe RL algorithms such as BCQ-Lag and CPQ follow this structure, where the critics provide reward and safety signals and the actor optimizes a policy under these signals. We build our unlearning method on top of this general formulation.

\textbf{Critic Unlearning.}
We first introduce unlearning for the critic networks, as they provide the value estimates that directly guide the actor updates. If the critic outputs on $\mathcal{D}_f$ still reflect the poisoned behavior, the actor will continue to recover the corrupted policy regardless of how the actor objective is designed.

For the reward critic $Q_\phi^r$ and cost critic $Q_\psi^c$, we adopt a dual-objective design that balances knowledge retention on clean data with information removal from poisoned data. On the keep set $\mathcal{D}_k$, we maintain standard temporal difference learning to preserve accurate value estimates:
\begin{equation} 
\mathcal{L}_{\text{critic}}^r(\mathcal{D}_k) 
= \mathbb{E}_{(s,a) \sim \mathcal{D}_k}\left[(Q_\phi^r(s,a) - y_r)^2\right],
\end{equation} 
\begin{equation} 
\mathcal{L}_{\text{critic}}^c(\mathcal{D}_k) 
= \mathbb{E}_{(s,a) \sim \mathcal{D}_k}\left[(Q_\psi^c(s,a) - y_c)^2\right],
\end{equation} 
where $y_r$ and $y_c$ denote the Bellman backup targets defined by the offline safe RL algorithm.

On the forget set $\mathcal{D}_f$, directly applying Bellman regression is not suitable for unlearning. Although one may consider reversing the Bellman update (e.g., minimizing the value), such operations still enforce a consistent value structure on $\mathcal{D}_f$. This means the critic continues to encode information from the poisoned data, even if the value magnitude is shifted. In addition, forcing the critic to deviate significantly from its Bellman targets may introduce instability and harm convergence.

Therefore, instead of following Bellman regression, we replace it with a value suppression objective. For the reward critic, we push the predicted values toward a lower reference $\bar{Q}^r$, and for the cost critic, toward a higher reference $\bar{Q}^c$:
\begin{equation}
\mathcal{L}_{\text{critic}}^r(\mathcal{D}_f)
  = \mathbb{E}_{(s,a) \sim \mathcal{D}_f}\!\left[
        \mathrm{softplus}\!\left(Q_\phi^r(s,a) - \bar{Q}^r\right)
    \right],
\end{equation}
\begin{equation}
\mathcal{L}_{\text{critic}}^c(\mathcal{D}_f)
  = \mathbb{E}_{(s,a) \sim \mathcal{D}_f}\!\left[
        \mathrm{softplus}\!\left(\kappa + \sigma - Q_\psi^c(s,a)\right)
    \right].
\end{equation}

The key intuition is that we only suppress the critic outputs when they are still influenced by the poisoned data. Once the predicted values move past the reference, the penalty gradually diminishes, avoiding unnecessary updates and preventing value collapse.  

To further improve stability, we select the reference values adaptively from the current training distribution. Specifically, $\bar{Q}^r$ is chosen as a low quantile (we use the 30th percentile) of the reward backup $y_r$ computed on $\mathcal{D}_k$. This ensures that the suppressed reward values remain below typical safe returns without collapsing to a trivial constant. For the cost critic, we use the cost threshold $\kappa$ with a small margin $\sigma$ to ensure that forget-state actions are consistently treated as violating the safety constraint, where $\sigma$ is a positive constant.

The combined critic loss is:
\begin{align} 
\mathcal{L}_{\text{critic}}^r &= \alpha_k \mathcal{L}_{\text{critic}}^r(\mathcal{D}_k) + \alpha_f \mathcal{L}_{\text{critic}}^r(\mathcal{D}_f), \\
\mathcal{L}_{\text{critic}}^c &= \alpha_k \mathcal{L}_{\text{critic}}^c(\mathcal{D}_k) + \alpha_f \mathcal{L}_{\text{critic}}^c(\mathcal{D}_f),
\end{align} 
where $\alpha_k$ and $\alpha_f$ control the relative importance of keeping and forgetting.

\textbf{Actor Unlearning.}
Given the updated critics, we unlearn the actor $\pi_\theta$ using policy gradient updates that combine objectives from both $\mathcal{D}_k$ and $\mathcal{D}_f$.

On the keep data, the actor maximizes reward while respecting the cost constraint and remaining close to safe behavioral actions:
\begin{equation}
\begin{aligned}
\mathcal{L}_{\text{actor}}(\mathcal{D}_k) =
& -\mathbb{E}_{\mathcal{D}_k}\!\left[\mathbbm{1}[Q^c_\psi \le \kappa]\,Q_\phi^r(s,\pi_\theta(s))\right] \\
& + \mathbb{E}_{\mathcal{D}_k}\!\left[\mathrm{softplus}(Q_\psi^c(s,\pi_\theta(s)) - \kappa)\right] 
\end{aligned}
\end{equation}

The first term follows the CPQ algorithm~\cite{xu2022constraints} safety-gated reward objective, where rewards are optimized only for actions that satisfy the cost constraint. The second term provides a smooth penalty when the cost exceeds the threshold, avoiding hard constraints.

On the forget set, we suppress the learned behavior by reversing the objective while coordinating with the cost critic:
\begin{equation}
\label{equ:actor df}
\begin{aligned}
\mathcal{L}_{\text{actor}}(\mathcal{D}_f) =
& \mathbb{E}_{\mathcal{D}_f}\!\left[\mathbbm{1}[Q^c_\psi \ge \kappa]\,Q_\phi^r(s,\pi_\theta(s))\right] \\
& + \mathbb{E}_{\mathcal{D}_f}\!\left[\mathrm{softplus}\!\left(\kappa + \sigma - Q_\psi^c(s,\pi_\theta(s))\right)\right].
\end{aligned}
\end{equation}

The first term in Equation \ref{equ:actor df} is the reward suppression term, which is gated by the cost critic, ensuring that the actor only suppresses reward when the state-action pair is confidently identified as unsafe. This design avoids instability caused by inaccurate cost estimates at the early stage of unlearning. The second term in Equation \ref{equ:actor df} is the cost suppression term that encourages the actor to move toward actions that are considered unsafe on the forget states, thereby shifting the policy distribution away from poisoned regions. 

\textbf{Adaptive Forget Weight $\beta_f$}. However, unlearning the forget set trajectory may induce instability by degrading the performance on the keep data. Therefore, we introduce an adaptive forget weight $\beta_f$ in the overall actor loss:
\begin{equation}
\mathcal{L}_{\text{actor}}^{\text{total}} 
= \mathcal{L}_{\text{actor}}(\mathcal{D}_k)
+ \beta_f\,\mathcal{L}_{\text{actor}}(\mathcal{D}_f),
\end{equation}
Where the $\beta_f$ controls the weight assigned to the forget objective. We update the $\beta_f$ based on the remaining forgetting gap, how far the forgotten data is considered not safe:
\[
\Delta = \kappa + \sigma - \mathbb{E}_{\mathcal{D}_f}[Q^c_\psi(s,\pi_\theta(s))].
\]
We update $\beta_f$ using a bounded increment:
\[
\beta_f \leftarrow \mathrm{clip}\!\left(\beta_f + \Delta*\tau_\beta,\; \beta_{\min},\; \beta_{\max}\right).
\]

where \(\tau_\beta\) is the update step size. When the forgetting is incomplete ($\Delta > 0$), $\beta_f$ increases to strengthen suppression. Once the gap closes, $\beta_f$ decreases, allowing the optimization to focus on retaining performance.

\textbf{Implementation on non actor-critic algorithms}. For algorithms without a standard actor-critic structure, such as COptiDICE~\cite{lee2022coptidice}, we implement unlearning by using the available value networks as proxies: the \(\nu\)-network provides the reward-side value weight, while the \(\chi\)-network serves as the cost-side substitution. The detailed implementation is illustrated in the Appendix.

%% file: Sections/06_Evaluation.tex
\vspace{-0.1in}
\section{EXPERIMENTS}
In this section, we demonstrate our experimental results for evaluating our safe unlearning method on different safe RL benchmarks.

\begin{table*}[t]
\centering
\scriptsize
\setlength{\tabcolsep}{3pt}
\renewcommand{\arraystretch}{1.15}

\resizebox{\textwidth}{!}{%
\begin{tabular}{@{}llc*{8}{c}@{}}
\toprule
\multirow{2}{*}{\textbf{Policy}}
& \multirow{2}{*}{\textbf{Attack}}
& \multirow{2}{*}{\textbf{Metric}}
& \multicolumn{2}{c}{\textbf{PointGoal}}
& \multicolumn{2}{c}{\textbf{CarCircle}}
& \multicolumn{2}{c}{\textbf{AntVelocity}}
& \multicolumn{2}{c}{\textbf{PointPush}} \\
\cmidrule(lr){4-5}\cmidrule(lr){6-7}\cmidrule(lr){8-9}\cmidrule(lr){10-11}
& & & \textbf{5\%} & \textbf{15\%} 
& \textbf{5\%} & \textbf{15\%} 
& \textbf{5\%} & \textbf{15\%} 
& \textbf{5\%} & \textbf{15\%} \\
\midrule

\multirow{6}{*}{BCQ-Lag} 
& \multirow{2}{*}{Max Cost} 
& \textbf{C} & 31.1 / \textbf{17.5} & 69.0 / \textbf{23.9} & 88.8 / \textbf{10.2} & 136.2 / \textbf{34.4} & 50.8 / \textbf{21.4} & 74.2 / \textbf{10.1} & 35.7 / \textbf{14.4} & 14.5 / \textbf{9.6} \\
& & \textbf{R} & 13.9 / \textbf{18.3} & 15.8 / 14.6 & 285.2 / \textbf{314.0} & 331.0 / \textbf{365.4} & 2984.9 / 2962.0 & 2898.4 / 1800.5 & 4.3 / 2.0 & -0.8 / \textbf{2.6} \\
\cmidrule(lr){2-11}

& \multirow{2}{*}{Max Reward} 
& \textbf{C} & 30.6 / \textbf{21.9} & 32.6 / \textbf{24.2} & 64.9 / \textbf{27.8} & 83.9 / \textbf{9.0} & 89.0 / \textbf{19.7} & 94.8 / \textbf{9.3} & 40.5 / \textbf{9.2} & 27.8 / \textbf{10.6} \\
& & \textbf{R} & 19.1 / \textbf{19.5} & 21.4 / 17.4 & 396.3 / 388.6 & 322.2 / \textbf{414.0} & 3018.9 / 2787.9 & 2603.1 / 1829.1 & 3.6 / 2.9 & 3.2 / 1.6 \\
\cmidrule(lr){2-11}

& \multirow{2}{*}{Min Reward} 
& \textbf{C} & 33.6 / \textbf{9.4} & 18.9 / \textbf{10.8} & 62.5 / \textbf{22.4} & 27.1 / \textbf{21.7} & 123.9 / \textbf{36.2} & 5.1 / 6.8 & 82.8 / \textbf{9.8} & 38.9 / \textbf{9.7} \\
& & \textbf{R} & 13.6 / 9.1 & 2.8 / \textbf{10.8} & 251.1 / \textbf{312.8} & 130.6 / \textbf{301.4} & 2781.7 / 2570.8 & 2678.4 / 2135.2 & 3.7 / \textbf{3.8} & -6.9 / \textbf{3.4} \\
\midrule

\multirow{6}{*}{CPQ} 
& \multirow{2}{*}{Max Cost} 
& \textbf{C} & 42.1 / \textbf{4.7} & 42.1 / \textbf{21.9} & 52.4 / \textbf{0.0} & 195.9 / \textbf{0.7} & 0.0 / 0.4 & 0.0 / 0.1 & 3.5 / \textbf{1.4} & 0.0 / 0.0 \\
& & \textbf{R} & 0.1 / \textbf{12.6} & 0.1 / 0.1 & 366.7 / \textbf{381.3} & 173.8 / \textbf{345.2} & -3000.1 / \textbf{932.9} & -2999.7 / \textbf{1966.5} & -1.7 / \textbf{1.9} & -12.0 / \textbf{0.1} \\
\cmidrule(lr){2-11}

& \multirow{2}{*}{Max Reward} 
& \textbf{C} & 32.0 / \textbf{7.2} & 4.9 / \textbf{3.5} & 0.0 / 0.0 & 0.0 / 0.0 & 0.0 / 1.5 & 0.0 / 1.5 & 3.9 / \textbf{0.0} & 11.4 / \textbf{2.4} \\
& & \textbf{R} & -0.7 / \textbf{0.7} & -2.1 / \textbf{7.9} & 284.6 / \textbf{381.1} & 364.9 / \textbf{380.3} & -2998.6 / \textbf{1242.3} & -2999.4 / \textbf{1403.7} & -23.6 / \textbf{1.5} & -1.5 / \textbf{1.9} \\
\cmidrule(lr){2-11}

& \multirow{2}{*}{Min Reward} 
& \textbf{C} & 52.9 / \textbf{9.4} & 54.6 / \textbf{5.3} & 134.8 / \textbf{0.0} & 73.9 / \textbf{4.4} & 0.0 / 8.2 & 0.0 / 0.0 & 0.8 / 6.4 & 2.0 / 8.5 \\
& & \textbf{R} & -1.9 / \textbf{14.4} & 0.3 / \textbf{0.5} & -50.2 / \textbf{306.6} & -159.3 / \textbf{294.2} & -2997.2 / \textbf{1339.9} & -2999.0 / \textbf{1188.5} & -2.0 / \textbf{1.4} & -1.2 / \textbf{1.4} \\
\midrule

\multirow{6}{*}{BEAR-Lag} 
& \multirow{2}{*}{Max Cost} 
& \textbf{C} & 208.0 / \textbf{8.9} & 262.1 / \textbf{10.4} & 64.6 / \textbf{3.9} & 94.1 / \textbf{2.9} & 101.3 / \textbf{5.8} & 520.3 / \textbf{2.3} & 31.5 / \textbf{7.4} & 95.6 / \textbf{2.5} \\
& & \textbf{R} & 0.0 / \textbf{9.5} & 1.0 / \textbf{9.1} & 447.4 / 352.7 & 437.5 / 388.5 & 544.4 / \textbf{2517.8} & 2738.2 / 2534.7 & 0.6 / \textbf{2.2} & -3.4 / \textbf{0.7} \\
\cmidrule(lr){2-11}

& \multirow{2}{*}{Max Reward} 
& \textbf{C} & 69.8 / \textbf{8.1} & 30.8 / \textbf{10.1} & 56.2 / \textbf{9.8} & 58.9 / \textbf{2.0} & 60.5 / \textbf{5.7} & 0.1 / 6.7 & 13.4 / \textbf{8.6} & 11.8 / \textbf{0.0} \\
& & \textbf{R} & -0.6 / \textbf{4.9} & 6.3 / \textbf{12.3} & 333.1 / \textbf{356.0} & 337.9 / 313.8 & 2660.2 / 2531.5 & 839.6 / \textbf{2493.0} & 0.4 / \textbf{2.2} & -18.2 / \textbf{2.9} \\
\cmidrule(lr){2-11}

& \multirow{2}{*}{Min Reward} 
& \textbf{C} & 27.9 / \textbf{4.9} & 6.8 / 7.3 & 79.1 / \textbf{7.0} & 68.3 / \textbf{5.9} & 0.0 / 8.5 & 0.0 / 0.0 & 11.7 / \textbf{4.2} & 28.5 / \textbf{7.6} \\
& & \textbf{R} & 10.5 / \textbf{12.1} & -26.8 / \textbf{8.4} & 466.6 / 363.9 & 199.0 / \textbf{321.6} & -309.4 / \textbf{2616.2} & -485.4 / \textbf{2289.6} & -11.2 / \textbf{1.8} & -10.9 / \textbf{1.6} \\
\midrule

\multirow{6}{*}{COptiDICE} 
& \multirow{2}{*}{Max Cost} 
& \textbf{C} & 64.3 / \textbf{17.1} & 45.3 / \textbf{22.7} & 94.1 / \textbf{38.3} & 90.7 / \textbf{10.6} & 108.0 / \textbf{7.3} & 39.8 / \textbf{5.7} & 31.5 / \textbf{5.8} & 37.7 / \textbf{10.3} \\
& & \textbf{R} & 14.2 / 12.7 & 16.7 / 8.7 & 190.9 / 175.0 & 209.8 / 176.9 & 2941.9 / 2416.0 & 1260.6 / \textbf{2501.5} & 1.6 / 1.2 & 0.3 / \textbf{0.4} \\
\cmidrule(lr){2-11}

& \multirow{2}{*}{Max Reward} 
& \textbf{C} & 60.6 / \textbf{10.3} & 70.3 / \textbf{14.9} & 75.2 / \textbf{11.5} & 97.2 / \textbf{19.4} & 90.2 / \textbf{8.9} & 46.1 / \textbf{7.8} & 28.7 / \textbf{8.2} & 53.0 / \textbf{6.0} \\
& & \textbf{R} & 15.6 / 5.9 & 13.0 / 7.8 & 262.6 / 167.6 & 169.9 / 144.9 & 2942.6 / 2319.1 & 1533.1 / \textbf{2428.2} & -0.1 / \textbf{0.3} & -0.4 / -2.7 \\
\cmidrule(lr){2-11}

& \multirow{2}{*}{Min Reward} 
& \textbf{C} & 46.4 / \textbf{10.4} & 36.6 / \textbf{2.3} & 69.3 / \textbf{13.1} & 84.4 / \textbf{18.6} & 140.2 / \textbf{0.0} & 43.9 / \textbf{2.4} & 45.3 / \textbf{7.5} & 34.4 / \textbf{9.8} \\
& & \textbf{R} & 9.0 / 4.9 & -9.1 / \textbf{1.6} & 239.0 / 157.7 & -64.7 / \textbf{147.0} & 2997.3 / 2257.8 & 851.9 / \textbf{2380.0} & 2.2 / 1.2 & -11.6 / \textbf{-1.0} \\

\bottomrule
\end{tabular}%
}
\caption{Cost and reward before/after unlearning under different attacks and poison ratios. Each entry reports before/after unlearning. Bold post-unlearning values indicate improvement: lower cost for \textbf{C}, higher reward for \textbf{R}.}
\vspace{-0.2in}
\label{tab:compact_unlearning}

\end{table*}

\textbf{Benchmarks.} Our experiments are conducted on four benchmarks in the Safety-Gymnasium \cite{ji2023safety}, which provides safe RL benchmarks targeting the challenge of safe RL. For the offline safe RL dataset, we use the OSRL \cite{liu2023datasets} dataset, which is an open-source offline safe RL dataset. The four benchmarks used in our experiments are CarCircle, PointGoal, AntVelocity, and PointPush.





\textbf{Attack Settings.}
To demonstrate that our unlearning method is effective in mitigating various offline dataset poisoning attacks, we consider the following attack strategies:

\emph{Max Cost and Max Reward attack.} 
Max Cost and Max Reward~\cite{liu2022robustness} are attacks that generate adversarial trajectories by maximizing the cost critic and reward critic values, respectively. To adapt these attacks to the offline setting, we first train an adversarial policy that produces high-cost or high-reward trajectories (the high-reward trajectories discard the cost constraints, which also leads to high cost).

\emph{Min Reward.}
To illustrate that our unlearning method can also mitigate attacks that do not directly target the cost, we additionally consider a reward-minimization attack. Similar to the Max Cost/Reward setting, we train an adversarial policy that generates trajectories with low rewards. 

To make these poisoned trajectories appear beneficial during training, we modify their labels by setting the cost to zero while keeping the reward at a high level comparable to the clean data. As a result, these trajectories are treated as highly valuable by the learning algorithm, which misleads the policy toward unsafe behaviors. These trajectories are then injected into the offline dataset to get $\tilde{\mathcal{D}}=\mathcal{D}_k\cup \mathcal{D}_f$.

\textbf{Defense Baselines and Ablations.}
To evaluate the effectiveness of our safe unlearning method against poisoning attacks, we compare it with the following defense baselines and ablation variants. {For simplicity, we use \emph{unlearning} to present our Safe-RULE method.}

\emph{Trajdeleter.}
We follow Trajdeleter~\cite{gong2024trajdeleter}, which removes the influence of poisoned trajectories through a two-phase procedure: a forgetting phase followed by a convergence phase. The original Trajdeleter is designed to unlearn reward-related poisoning effects. In the Appendix, we further adapt Trajdeleter to unlearn cost-related poisoning effects for a fairer comparison with our method.

\emph{Fine-tuning.}
We evaluate a simple fine-tuning baseline that updates the corrupted policy using only clean data. This baseline examines whether standard clean-data fine-tuning is sufficient to remove the influence of poisoned trajectories.

\emph{Reward-only unlearning.}
Our safe unlearning method jointly addresses reward and cost corruption. To isolate the effect of cost unlearning, we include a reward-only variant of our method. This variant uses the same unlearning framework but removes the cost-related unlearning terms from both the critic and actor objectives.

\emph{Training from scratch.}
We also report the performance of clean policies trained from scratch using only the clean dataset. Each policy is trained for \(10^7\) steps for each algorithm and environment. This setting serves as a clean reference to show how closely the unlearned policy can recover the performance of a policy trained without poisoned data.

\vspace{-0.1in}
\subsection{Experimental Settings}

\textbf{Training setup.} We follow the OSRL\cite{liu2023datasets} setup for offline safe RL, including the offline safe RL algorithms and datasets. We train four algorithms: BCQ-Lag~\cite{fujimoto2019off}, BEAR-Lag~\cite{kumar2019stabilizing}, COptiDICE~\cite{lee2022coptidice}, and CPQ~\cite{xu2022constraints} on the corrupted datasets $\tilde{\mathcal{D}}$ produced by the attacks above. Each run uses $10^{6}$ gradient steps, discount factor $\gamma=0.99$. The detailed hyperparameter settings, including training and unlearning, are listed in the Appendix.

\textbf{Adversary setup.}
We consider three poisoning strategies to construct the poisoned dataset \(\mathcal{D}_f\): \emph{Max Cost}, \emph{Max Reward}, and \emph{Min Reward}. We vary the poisoning ratio $\rho = |\mathcal{D}_f|/|\tilde{\mathcal{D}}|\in\{0.05, 0.15\}$. Each strategy uses a corresponding adversarial policy to generate trajectories \((s,a,r,c)\) with targeted properties: high cost, high reward, or low reward. Training on this hybrid dataset can therefore cause the safe RL policy to learn undesirable behaviors, leading to degraded reward or increased cost.

\vspace{-0.1in}
\subsection{Results}

\begin{figure*}[t]
\centering

\begin{subfigure}{0.99\textwidth}
    \centering
    \includegraphics[width=\textwidth]{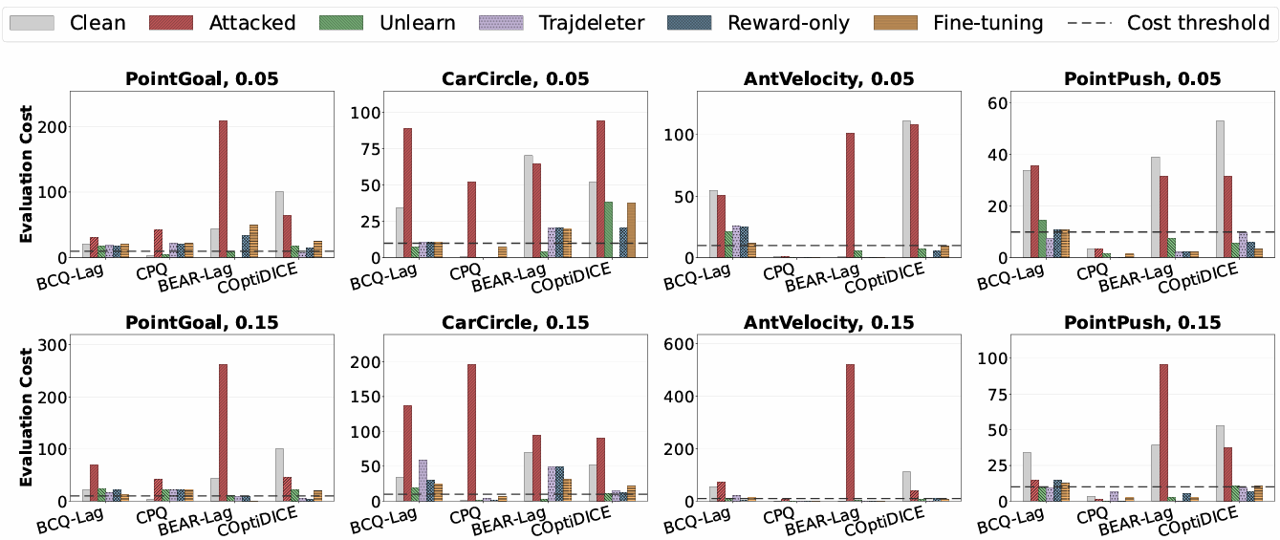}
\caption{Evaluation \textbf{cost} under the Max Cost attack across different environments and safe RL algorithms with poisoning ratios of \(5\%\) and \(15\%\). The green bars represent the cost after applying our safe unlearning method; lower values indicate better performance.}
    \label{fig:cost_max_cost}
\end{subfigure}

\vspace{0.08in}

\begin{subfigure}{0.99\textwidth}
    \centering
    \includegraphics[width=\textwidth]{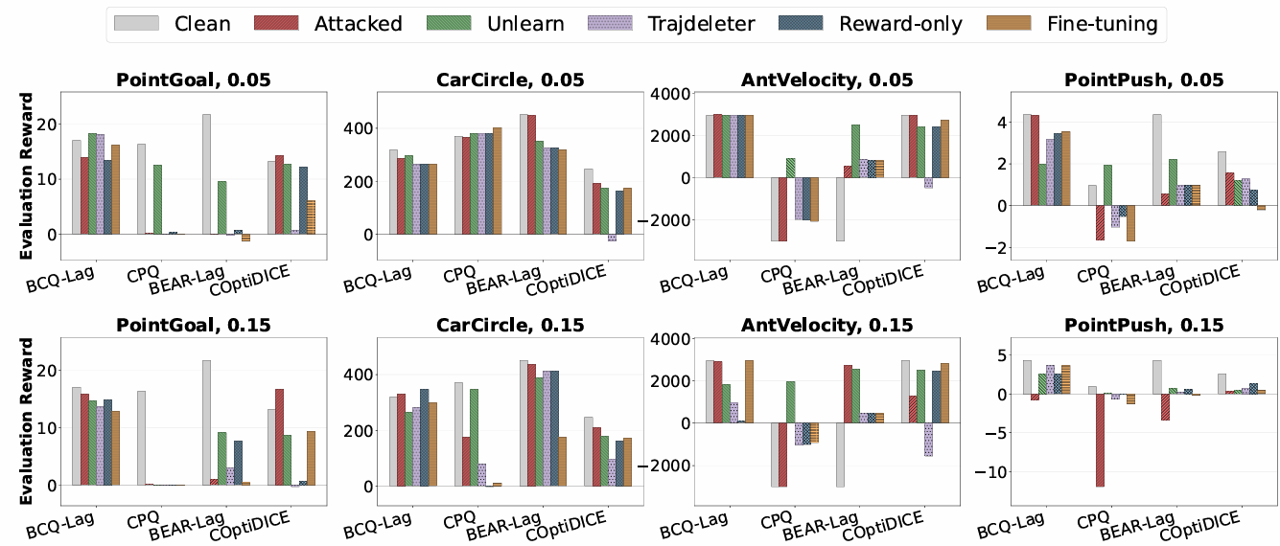}
\caption{Evaluation \textbf{reward} under the Max Cost attack across different environments and safe RL algorithms with poisoning ratios of \(5\%\) and \(15\%\). The green bars represent the reward after applying our safe unlearning method; higher values indicate better performance.}
    \label{fig:reward_max_cost}
\end{subfigure}
\vspace{-0.2in}
\label{fig:cost_unlearning_results}
\end{figure*}

We evaluate the effectiveness of our unlearning method in mitigating poison attacks by comparing evaluation cost and reward values across benchmarks and safe RL algorithms.

\textbf{Impact of attacks and safe unlearning.}
In Table \ref{tab:compact_unlearning}, we report the evaluation cost and reward of the attacked policies before and after applying our safe unlearning method under three poisoning attacks. Each entry shows the performance before/after 5000 unlearning steps. The bold post-unlearning value indicates an improvement, i.e., a lower cost for \textbf{C} or a higher reward for \textbf{R}. Overall, safe unlearning reduces cost or improves reward in many settings across different environments, attacks, and safe RL algorithms, demonstrating its general applicability. We also observe that some cases still involve reward-cost trade-offs, indicating that safe unlearning must balance performance recovery and safety recovery.

In Figure \ref{fig:reward_max_cost} and Figure \ref{fig:cost_max_cost}, we demonstrate the evaluation cost and reward under the Max Cost attack with four different offline safe RL algorithms. As shown, the Max Cost attack substantially increases the evaluation cost and degrades policy performance. The results show that our safe unlearning algorithm achieves the best performance among all baselines, outperforming both Trajdeleter and fine-tuning approaches, and has a slight performance drop compared with the clean policy.





\textbf{Effect of adaptive $\beta_f$.}
To balance forgetting and retention during safe unlearning, we use an adaptive $\beta_f$ to control the weight of the forget-gradient update. This design prevents the model from overfitting to the forget dataset while preserving useful knowledge from the keep dataset. Table~\ref{tab:eta_forget_ablation_5poison} reports the changes in evaluation cost and reward using adaptive $\beta_f$ and a fixed setting $\beta_f=1$. Here, $\beta_f=1$ means that the keep-gradient and forget-gradient terms are assigned the same weight in the actor update. The results show that adaptive $\beta_f$ achieves better performance in most cases, especially in retaining reward while maintaining competitive cost reduction.

\begin{table*}[h]
\centering
\tiny
\setlength{\tabcolsep}{3pt}
\renewcommand{\arraystretch}{1.08}
\resizebox{\textwidth}{!}{
\begin{tabular}{llcccccccc}
\toprule
\multirow{2}{*}{\textbf{Policy}}
& \multirow{2}{*}{\textbf{Variant}}
& \multicolumn{2}{c}{\textbf{PointGoal}}
& \multicolumn{2}{c}{\textbf{CarCircle}}
& \multicolumn{2}{c}{\textbf{AntVelocity}}
& \multicolumn{2}{c}{\textbf{PointPush}} \\
\cmidrule(lr){3-4}\cmidrule(lr){5-6}\cmidrule(lr){7-8}\cmidrule(lr){9-10}
& & $\Delta C\uparrow$ & $\Delta R\uparrow$
& $\Delta C\uparrow$ & $\Delta R\uparrow$
& $\Delta C\uparrow$ & $\Delta R\uparrow$
& $\Delta C\uparrow$ & $\Delta R\uparrow$ \\
\midrule

\multirow{2}{*}{BCQ-Lag}
& adaptive $\beta_f$
& \textbf{13.6} & \textbf{4.4}
& 81.2 & \textbf{10.8}
& 29.4 & \textbf{-22.9}
& 21.3 & -2.3 \\
& $\beta_f=1$
& 8.6 & -1.2
& \textbf{81.9} & -15.6
& \textbf{31.6} & -739.1
& \textbf{26.9} & \textbf{-1.5} \\

\midrule
\multirow{2}{*}{CPQ}
& adaptive $\beta_f$
& \textbf{37.4} & \textbf{12.5}
& 52.4 & 14.6
& -0.4 & \textbf{3933.0}
& \textbf{2.1} & \textbf{3.6} \\
& $\beta_f=1$
& 33.0 & 9.5
& 52.4 & 14.6
& \textbf{0.0} & 3910.3
& 1.0 & 3.5 \\

\midrule
\multirow{2}{*}{BEAR-Lag}
& adaptive $\beta_f$
& \textbf{199.1} & \textbf{9.5}
& 55.8 & \textbf{-202.5}
& 95.5 & \textbf{1973.4}
& 24.1 & \textbf{1.6} \\
& $\beta_f=1$
& 185.4 & 0.1
& 55.8 & -253.3
& \textbf{101.3} & 300.4
& \textbf{29.1} & 0.4 \\

\midrule
\multirow{2}{*}{COptiDICE}
& adaptive $\beta_f$
& 47.2 & \textbf{-1.6}
& 55.8 & \textbf{-16.0}
& 100.7 & \textbf{-525.9}
& \textbf{25.7} & \textbf{-0.4} \\
& $\beta_f=1$
& \textbf{49.8} & -6.9
& \textbf{67.9} & -20.4
& \textbf{108.0} & -632.7
& 23.4 & -0.5 \\

\bottomrule
\end{tabular}
}
\caption{Effect of adaptive forget-gradient weight under Max Cost attack with 5\% poison. Each entry reports the improvement relative to the poisoned policy, where $\Delta C=C_{\mathrm{poisoned}}-C_{\mathrm{unlearn}}$ and $\Delta R=R_{\mathrm{unlearn}}-R_{\mathrm{poisoned}}$. Bold values indicate the better variant between adaptive and fixed $\beta_f$.}
\vspace{-0.15in}
\label{tab:eta_forget_ablation_5poison}
\end{table*}

\textbf{Time cost of unlearning.}
In Table \ref{tab:runtime_cpq_maxcost15}, we compare the runtime of our safe unlearning method with baseline defenses and training from scratch on clean data. Since both unlearning and the baseline defenses require only a small number of update steps, they are much more efficient than retraining a clean policy from scratch. This computational efficiency is a key advantage of unlearning-based defenses, especially when the original training process requires a large number of environment or optimization steps.

\begin{table}[t]
\centering
\small
\setlength{\tabcolsep}{4pt}
\renewcommand{\arraystretch}{1.1}
\begin{tabular}{lcccc}
\toprule
\multirow{2}{*}{\textbf{Method}} & \multicolumn{4}{c}{\textbf{Task}} \\
\cmidrule(lr){2-5}
& PointGoal & CarCircle & AntVelocity & PointPush \\
\midrule
Retraining      & 132.0 & 276.1 & 252.4 & 212.3 \\
Safe Unlearning &  5.5 &   9.2 &   7.3 &   6.3 \\
Trajdeleter     &  3.1 &   6.6 &   5.4 &   3.4 \\
Reward-only     &  3.1 &   6.4 &   5.3 &   3.2 \\
Fine-tuning     &  3.6 &   7.0 &   6.2 &   5.5 \\
\bottomrule
\end{tabular}
\caption{Computation overhead in minutes using CPQ. }
\vspace{-0.2in}
\label{tab:runtime_cpq_maxcost15}
\end{table}

%% file: Sections/07_conclusion.tex
\vspace{-0.1in}
\section{Limitations}
Although Safe-RULE provides an effective defense for offline safe RL against poisoning attacks, it still has several limitations. First, our method assumes that the keep dataset \(\mathcal{D}_k\) and forget dataset \(\mathcal{D}_f\) can be identified or separated. If poisoned samples cannot be detected or isolated, Safe-RULE cannot directly remove their influence. Second, while Safe-RULE can reduce the impact of poisoning attacks, the unlearned policy may not fully recover the reward and cost performance of a clean policy trained only on clean data. Third, our current experiments are conducted in simulation environments, and further evaluation on real robotic or autonomous systems is needed to validate the practical effectiveness of the method.

\vspace{-0.1in}
\section{Conclusion}

In this paper, we proposed \textbf{Safe Reinforcement Unlearning (Safe-RULE)} to mitigate poisoning attacks in offline safe reinforcement learning. Safe-RULE jointly unlearns the critic and actor networks to remove the influence of poisoned data while preserving useful knowledge from the keep dataset. To balance forgetting and retention, we introduce an adaptive forget weight that controls the forget-gradient update and reduces over-forgetting. Different from standard unlearning, safe reinforcement unlearning must recover task performance while satisfying the safety constraint on the clean distribution. Experiments on offline safe RL benchmarks show that Safe-RULE effectively reduces poisoning effects and recovers safer policy behavior without retraining from scratch.


%% file: Sections/08_Appendix.tex
\newpage
\appendix
\section{Theoretical Analysis}
\label{sec:reward_reference_cost_margin}

In this section, we provide a theoretical explanation for two design choices in
Safe-RULE: the reward reference $\bar{Q}^r$ and the safety margin $\sigma$ in
the forget-cost target $\kappa+\sigma$. We show that $\bar{Q}^r$ is needed to
remove the reward-side attraction of poisoned samples, while $\sigma$ is needed
to obtain a non-vacuous separation between forget samples and the safety
threshold $\kappa$.

Recall that the critic forget objectives are
\begin{equation}
\label{eq:reward_ref_loss}
\mathcal{L}_{\mathrm{critic}}^r(\mathcal{D}_f)
=
\mathbb{E}_{(s,a)\sim\mathcal{D}_f}
\left[
\mathrm{softplus}
\left(
Q_\phi^r(s,a)-\bar{Q}^r
\right)
\right],
\end{equation}
and
\begin{equation}
\label{eq:cost_margin_loss}
\mathcal{L}_{\mathrm{critic}}^c(\mathcal{D}_f)
=
\mathbb{E}_{(s,a)\sim\mathcal{D}_f}
\left[
\mathrm{softplus}
\left(
\kappa+\sigma-Q_\psi^c(s,a)
\right)
\right].
\end{equation}

\paragraph{Reward reference.}
We first show that the reward reference $\bar{Q}^r$ gives an explicit upper
bound on the reward-side influence of the forget data. This is important
because poisoned trajectories may still attract the actor if their predicted
reward remains high, even when their cost values are corrected.

\begin{lemma}[Reward reference suppresses poisoned reward]
\label{lem:reward_reference}
Assume that after critic unlearning, the unlearning reward critic loss for the forget data satisfies:
\[
\mathcal{L}_{\mathrm{critic}}^r(\mathcal{D}_f) \le \varepsilon_r ,
\]
where $\varepsilon_r \ge 0$ is the reward critic residual. Then the reward excess above $\bar{Q}^r$ on the forget set is bounded by
\begin{equation}
\mathbb{E}_{(s,a)\sim\mathcal{D}_f}
\left[
\left(Q_\phi^r(s,a)-\bar{Q}^r\right)_+
\right]
\le
\varepsilon_r .
\end{equation}
Moreover, for any $u>0$,
\begin{equation}
\Pr_{(s,a)\sim\mathcal{D}_f}
\left(
Q_\phi^r(s,a) \ge \bar{Q}^r + u
\right)
\le
\frac{\varepsilon_r}{u}.
\end{equation}
\end{lemma}

\begin{proof}
For any $x\in\mathbb{R}$, we have
\[
\mathrm{softplus}(x)=\log(1+e^x)\ge (x)_+ .
\]
Applying this inequality to
$x=Q_\phi^r(s,a)-\bar{Q}^r$ gives
\begin{equation}
\mathbb{E}_{(s,a)\sim\mathcal{D}_f}
\left[
\left(Q_\phi^r(s,a)-\bar{Q}^r\right)_+
\right]
\le
\mathcal{L}_{\mathrm{critic}}^r(\mathcal{D}_f)
\le
\varepsilon_r .
\end{equation}
For the second claim, by Markov's inequality,
\begin{align}
\Pr_{(s,a)\sim\mathcal{D}_f}
\left(
Q_\phi^r(s,a)\ge \bar{Q}^r+u
\right)
&=
\Pr_{(s,a)\sim\mathcal{D}_f}
\left(
\left(Q_\phi^r(s,a)-\bar{Q}^r\right)_+ \ge u
\right) \\
&\le
\frac{
\mathbb{E}_{(s,a)\sim\mathcal{D}_f}
\left[
\left(Q_\phi^r(s,a)-\bar{Q}^r\right)_+
\right]
}{u}
\le
\frac{\varepsilon_r}{u}.
\end{align}
\end{proof}

Lemma~\ref{lem:reward_reference} shows that once the critic reward forget loss is
small, poisoned samples cannot keep large reward values above the reference
$\bar{Q}^r$. In our method, $\bar{Q}^r$ is chosen as a low quantile of the
reward backup $y_r$ computed on $\mathcal{D}_k$. Therefore, the reward critic
does not collapse forget samples to an arbitrary constant; instead, it pushes
their reward values below a typical clean-data return level.

This also explains why the reward reference is necessary. If the reward
suppression term is removed, then the forget objective no longer constrains
$Q_\phi^r$ on $\mathcal{D}_f$. In that case, even if the cost critic marks the
forget samples as unsafe, the reward critic can still assign high reward to
them. The actor may then continue to receive reward-side gradients that pull it
back toward the poisoned behavior. The reference $\bar{Q}^r$ prevents this
failure by explicitly bounding the reward-side attraction of $\mathcal{D}_f$.

\paragraph{Cost margin.}
We next show why the forget-cost target should be $\kappa+\sigma$ rather than
only $\kappa$. The margin $\sigma$ gives a strict separation between forget
samples and the safety threshold.

\begin{lemma}[Cost margin gives non-vacuous safety separation]
\label{lem:cost_margin}
Assume that after critic unlearning,
\[
\mathcal{L}_{\mathrm{critic}}^c(\mathcal{D}_f) \le \varepsilon_c .
\]
Then the probability that forget samples are still predicted as safe is bounded
by
\begin{equation}
\Pr_{(s,a)\sim\mathcal{D}_f}
\left(
Q_\psi^c(s,a) \le \kappa
\right)
\le
\frac{\varepsilon_c}{\sigma}.
\end{equation}
\end{lemma}

\begin{proof}
Again using $\mathrm{softplus}(x)\ge (x)_+$, we have
\begin{equation}
\mathbb{E}_{(s,a)\sim\mathcal{D}_f}
\left[
\left(\kappa+\sigma-Q_\psi^c(s,a)\right)_+
\right]
\le
\mathcal{L}_{\mathrm{critic}}^c(\mathcal{D}_f)
\le
\varepsilon_c .
\end{equation}
If $Q_\psi^c(s,a)\le\kappa$, then
\[
\kappa+\sigma-Q_\psi^c(s,a)\ge \sigma .
\]
Therefore,
\begin{align}
\Pr_{(s,a)\sim\mathcal{D}_f}
\left(
Q_\psi^c(s,a)\le\kappa
\right)
&\le
\Pr_{(s,a)\sim\mathcal{D}_f}
\left(
\left(\kappa+\sigma-Q_\psi^c(s,a)\right)_+ \ge \sigma
\right) \\
&\le
\frac{
\mathbb{E}_{(s,a)\sim\mathcal{D}_f}
\left[
\left(\kappa+\sigma-Q_\psi^c(s,a)\right)_+
\right]
}{\sigma}
\le
\frac{\varepsilon_c}{\sigma},
\end{align}
where the second inequality follows from Markov's inequality.
\end{proof}

Lemma~\ref{lem:cost_margin} shows that the margin $\sigma$ is essential for a
non-vacuous guarantee. If we only push the forget cost values to $\kappa$, then
$\sigma=0$ and the bound becomes meaningless. In this case, small critic errors
or actor updates can move the predicted cost back below $\kappa$, making the
forget samples appear safe again. By using $\kappa+\sigma$, Safe-RULE creates a
buffer above the safety threshold.

The margin also provides robustness in continuous action spaces. Suppose
$Q_\psi^c(s,\cdot)$ is $L_Q$-Lipschitz in the action, and after critic
unlearning the poisoned action satisfies
\begin{equation}
Q_\psi^c(s_f,a_f^{\mathrm{poison}})
\ge
\kappa+\sigma-\epsilon_c,
\end{equation}
where $\epsilon_c<\sigma$. Then for any action $a$ satisfying
\[
\|a-a_f^{\mathrm{poison}}\|
\le
\frac{\sigma-\epsilon_c}{L_Q},
\]
we have
\begin{align}
Q_\psi^c(s_f,a)
&\ge
Q_\psi^c(s_f,a_f^{\mathrm{poison}})
-
L_Q\|a-a_f^{\mathrm{poison}}\| \\
&\ge
\kappa+\sigma-\epsilon_c
-
L_Q\cdot
\frac{\sigma-\epsilon_c}{L_Q}
=
\kappa .
\end{align}
Thus, when $\sigma>\epsilon_c$, not only the poisoned action itself but also a
neighborhood around it is treated as unsafe. Without the margin, this certified
neighborhood has zero radius.

\paragraph{Combined implication.}
Combining Lemma~\ref{lem:reward_reference} and Lemma~\ref{lem:cost_margin},
if
\[
\mathcal{L}_{\mathrm{critic}}^r(\mathcal{D}_f)\le\varepsilon_r,
\qquad
\mathcal{L}_{\mathrm{critic}}^c(\mathcal{D}_f)\le\varepsilon_c,
\]
then Safe-RULE satisfies
\begin{align}
\mathbb{E}_{(s,a)\sim\mathcal{D}_f}
\left[
\left(Q_\phi^r(s,a)-\bar{Q}^r\right)_+
\right]
&\le
\varepsilon_r, \\
\Pr_{(s,a)\sim\mathcal{D}_f}
\left(
Q_\psi^c(s,a)\le\kappa
\right)
&\le
\frac{\varepsilon_c}{\sigma}.
\end{align}
The first inequality shows that the reward reference removes high reward
attraction from poisoned samples. The second inequality shows that the cost
margin makes poisoned samples reliably unsafe with respect to the threshold
$\kappa$. Therefore, $\bar{Q}^r$ and $\sigma$ are necessary to obtain
non-vacuous reward suppression and cost separation guarantees.

\section{Additional Experiments}

\textbf{Evaluation reward and cost for Min Reward and Max Reward attacks.}
We show the evaluation cost and reward across different environments and safe RL algorithms with poisoning ratios in Figure~\ref{fig:cost_max_cost} and Figure~\ref{fig:reward_max_cost}. Here, we continue to show the evaluation reward and cost for Min Reward and Max Reward poisoning attacks.

\begin{figure*}[h]
\centering
\begin{subfigure}{0.99\textwidth}
    \centering
    \includegraphics[width=\textwidth]{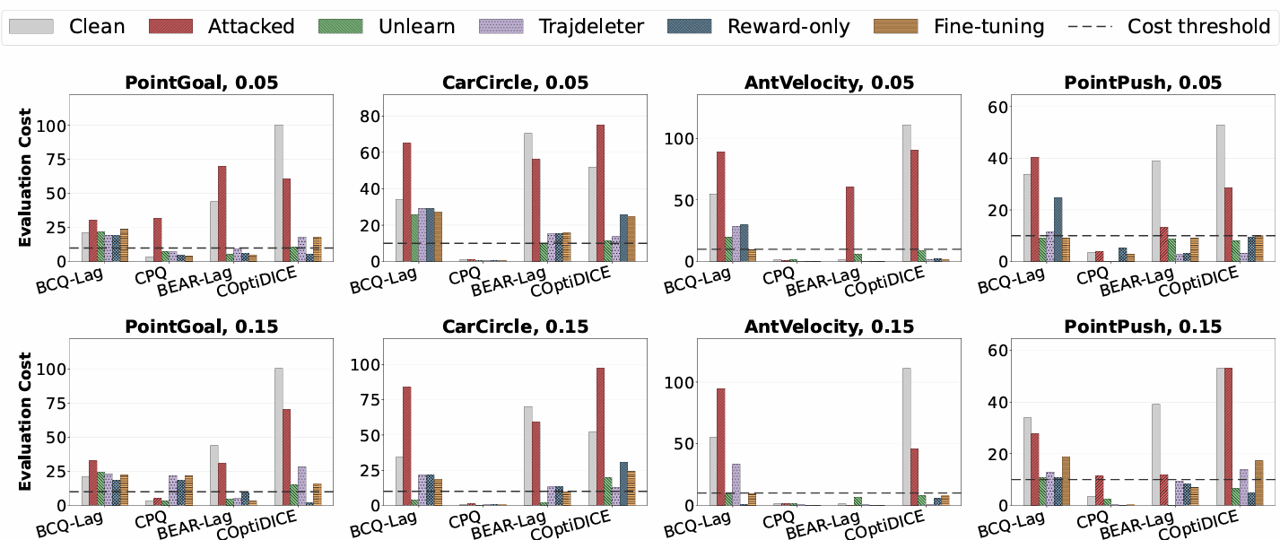}
    \caption{Evaluation \textbf{cost} under the Max Reward attack across different environments and safe RL algorithms with poisoning ratios of 5\% and 15\%. Lower values indicate better performance.}
    \label{fig:cost_max_reward}
\end{subfigure}

\vspace{0.08in}

\begin{subfigure}{0.99\textwidth}
    \centering
    \includegraphics[width=\textwidth]{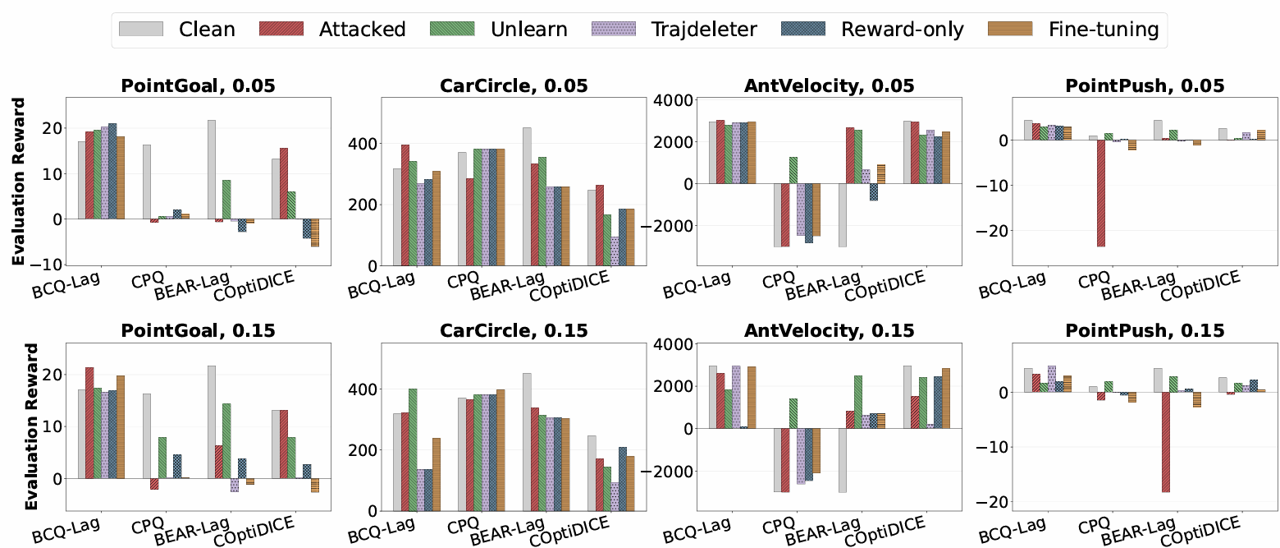}
    \caption{Evaluation \textbf{reward} under the Max Reward attack across different environments and safe RL algorithms with poisoning ratios of 5\% and 15\%. Higher values indicate better performance.}
    \label{fig:reward_max_reward}
\end{subfigure}

\vspace{-0.1in}
\caption{Evaluation cost and reward under the Max Reward poisoning attack. The results show that Safe-RULE can mitigate reward-targeted poisoning attacks while maintaining the reward--cost trade-off across different offline safe RL algorithms.}
\label{fig:evaluation_max_reward}
\end{figure*}

\begin{figure*}[h]
\centering
\begin{subfigure}{0.99\textwidth}
    \centering
    \includegraphics[width=\textwidth]{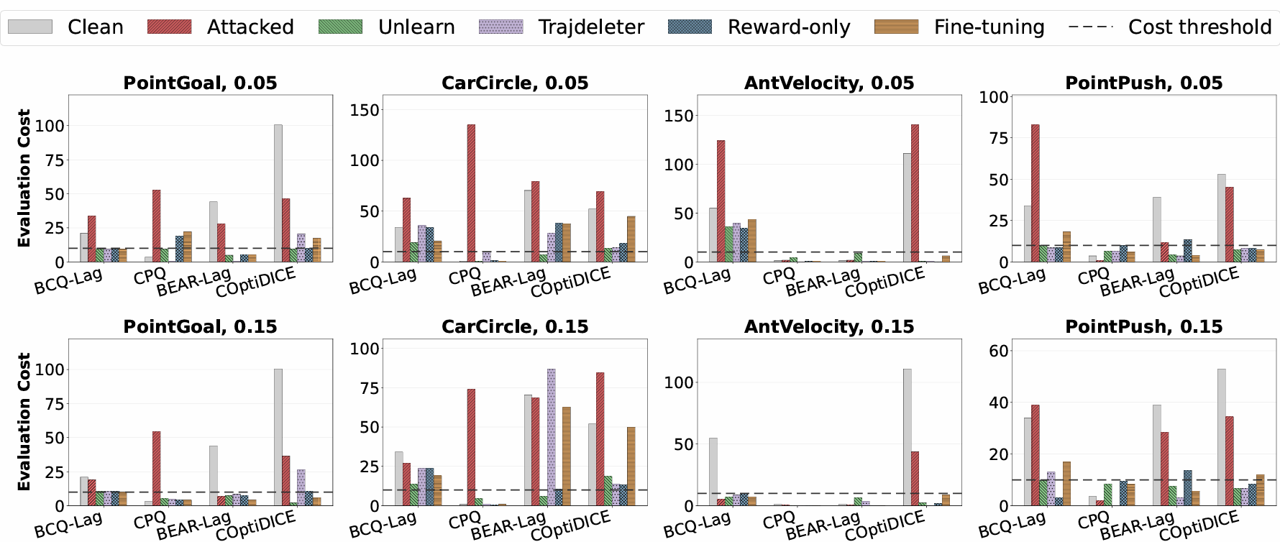}
    \caption{Evaluation \textbf{cost} under the Min Reward attack across different environments and safe RL algorithms with poisoning ratios of 5\% and 15\%. Lower values indicate better performance.}
    \label{fig:cost_min_reward}
\end{subfigure}

\vspace{0.08in}

\begin{subfigure}{0.99\textwidth}
    \centering
    \includegraphics[width=\textwidth]{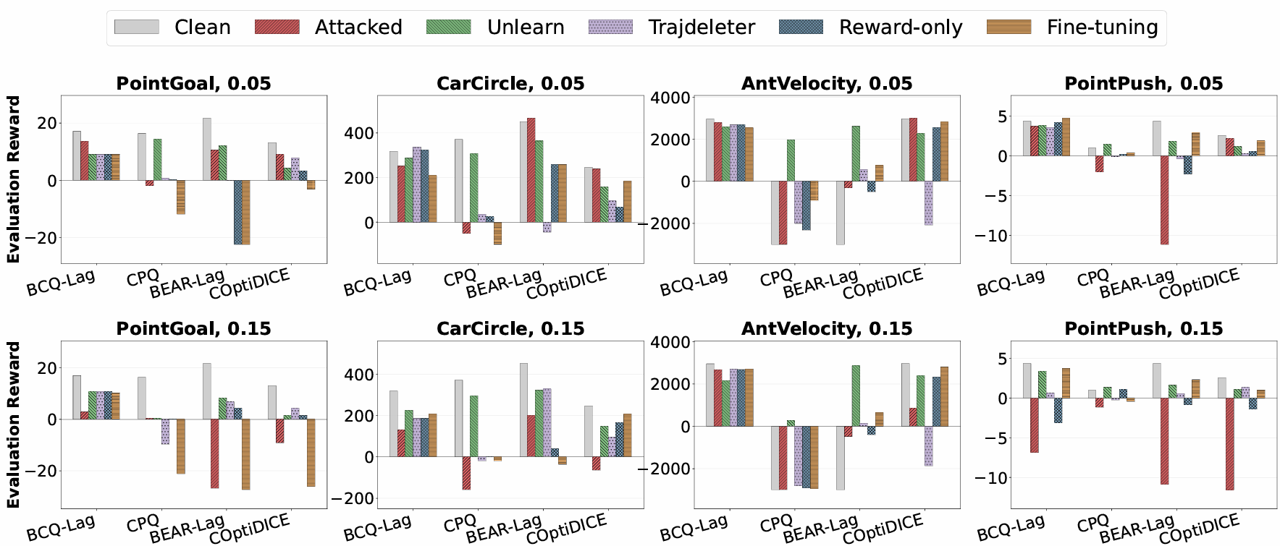}
    \caption{Evaluation \textbf{reward} under the Min Reward attack across different environments and safe RL algorithms with poisoning ratios of 5\% and 15\%. Higher values indicate better performance.}
    \label{fig:reward_min_reward}
\end{subfigure}

\vspace{-0.1in}
\caption{Evaluation cost and reward under the Min Reward poisoning attack. Although this attack does not directly maximize the cost, it can still degrade the learned policy by injecting low-reward trajectories with misleading labels. Safe-RULE improves policy recovery by jointly considering reward and cost during unlearning.}
\label{fig:evaluation_min_reward}
\end{figure*}

\textbf{Trajdeleter with cost.}
Trajdeleter~\cite{gong2024trajdeleter} unlearns the forget data by suppressing the advantage function of the forget data. This strategy can also be implemented with the cost advantage in safe RL. We therefore extend Trajdeleter as a cost-aware Trajdeleter baseline, where both reward and cost advantages are used to update the actor network of the safe policy.

Specifically, let
\[
A_r(s,a)=Q^r(s,a)-V^r(s),
\qquad
A_c(s,a)=Q^c(s,a)-V^c(s)
\]
denote the reward advantage and cost advantage, respectively. The original Trajdeleter mainly suppresses the reward advantage on the forget data. In the cost-aware version, we additionally use the cost advantage to make the forget data less desirable for the policy. The forget objective is written as
\begin{equation}
\mathcal{L}_{\mathrm{Traj}}(\mathcal{D}_f)
=
\mathbb{E}_{(s,a)\sim\mathcal{D}_f}
\left[
A_r(s,\pi_\theta(s))
+
\lambda_c\,
\mathrm{softplus}\!\left(\bar{A}_c - A_c(s,\pi_\theta(s))\right)
\right],
\end{equation}
where $\bar{A}_c$ is the target cost-advantage reference and $\lambda_c$ controls the cost-aware forgetting strength. The first term follows the reward-side forgetting objective of Trajdeleter, while the second term uses the cost signal to further move the actor away from poisoned regions.

The overall actor objective for cost-aware Trajdeleter is
\begin{equation}
\mathcal{L}_{\mathrm{Traj}}^{\mathrm{total}}
=
\mathcal{L}_{\mathrm{keep}}(\mathcal{D}_k)
+
\beta_f \mathcal{L}_{\mathrm{Traj}}(\mathcal{D}_f),
\end{equation}
where $\mathcal{L}_{\mathrm{keep}}(\mathcal{D}_k)$ is the original keep-data objective and $\beta_f$ controls the forgetting strength.

In Table~\ref{tab:trajdeleter_cost_aware_compare}, we show the results after applying the original Trajdeleter and cost-aware Trajdeleter. The cost-aware Trajdeleter improves the performance of the safe policy in most cases. However, its performance is still worse than Safe-RULE, since Safe-RULE jointly unlearns both critic and actor networks and directly suppresses the poisoned value structure.

\begin{table*}[h]
\centering
\scriptsize
\setlength{\tabcolsep}{3pt}
\renewcommand{\arraystretch}{1.15}
\resizebox{\textwidth}{!}{%
\begin{tabular}{@{}llc*{8}{c}@{}}
\toprule
\multirow{2}{*}{\textbf{Policy}}
& \multirow{2}{*}{\textbf{Attack}}
& \multirow{2}{*}{\textbf{Metric}}
& \multicolumn{2}{c}{\textbf{PointGoal}}
& \multicolumn{2}{c}{\textbf{CarCircle}}
& \multicolumn{2}{c}{\textbf{AntVelocity}}
& \multicolumn{2}{c}{\textbf{PointPush}} \\
\cmidrule(lr){4-5}\cmidrule(lr){6-7}\cmidrule(lr){8-9}\cmidrule(lr){10-11}
& & & \textbf{5\%} & \textbf{15\%}
& \textbf{5\%} & \textbf{15\%}
& \textbf{5\%} & \textbf{15\%}
& \textbf{5\%} & \textbf{15\%} \\
\midrule

\multirow{6}{*}{BCQ-Lag}
& \multirow{2}{*}{Max Cost}
& \textbf{C} & 19.9 / \textbf{7.4} & 17.0 / \textbf{16.3} & 10.4 / \textbf{1.7} & 59.2 / \textbf{25.8} & 25.5 / \textbf{7.7} & 22.8 / \textbf{17.9} & 7.5 / 12.6 & 9.4 / 10.5 \\
& & \textbf{R} & 18.2 / 15.4 & 13.7 / \textbf{16.6} & 264.1 / \textbf{272.3} & 280.6 / \textbf{287.0} & 2964.9 / 2785.6 & 951.5 / \textbf{1679.1} & 3.2 / \textbf{4.1} & 3.7 / 3.6 \\
\cmidrule(lr){2-11}
& \multirow{2}{*}{Max Reward}
& \textbf{C} & 19.4 / 26.4 & 23.1 / \textbf{21.8} & 29.3 / \textbf{16.7} & 21.5 / \textbf{16.3} & 28.7 / \textbf{9.0} & 33.2 / \textbf{31.4} & 11.5 / 13.3 & 12.9 / \textbf{9.7} \\
& & \textbf{R} & 20.3 / 20.2 & 16.6 / 12.3 & 268.7 / 234.9 & 135.2 / 120.3 & 2912.9 / 2873.9 & 2945.6 / \textbf{2952.0} & 3.3 / 2.6 & 4.8 / 2.5 \\
\cmidrule(lr){2-11}
& \multirow{2}{*}{Min Reward}
& \textbf{C} & 9.4 / \textbf{7.1} & 10.8 / \textbf{6.1} & 35.7 / \textbf{17.8} & 23.7 / \textbf{19.0} & 39.3 / \textbf{35.9} & 8.8 / \textbf{8.6} & 8.4 / 10.8 & 13.0 / \textbf{9.7} \\
& & \textbf{R} & 9.1 / \textbf{13.6} & 10.8 / -2.3 & 337.6 / 304.8 & 187.1 / \textbf{282.8} & 2693.3 / 2479.9 & 2709.9 / \textbf{2733.1} & 3.5 / 1.3 & 0.6 / \textbf{3.0} \\
\midrule

\multirow{6}{*}{CPQ}
& \multirow{2}{*}{Max Cost}
& \textbf{C} & 21.9 / 39.5 & 21.9 / \textbf{18.3} & 0.0 / 0.0 & 4.0 / 7.2 & 0.0 / 0.0 & 0.0 / 0.0 & 0.0 / 1.4 & 6.8 / \textbf{0.0} \\
& & \textbf{R} & 0.1 / 0.0 & 0.1 / -1.4 & 381.3 / 381.3 & 77.0 / 46.9 & -1980.0 / \textbf{-1763.4} & -1044.8 / -1046.1 & -1.0 / \textbf{-0.1} & -0.6 / \textbf{-0.1} \\
\cmidrule(lr){2-11}
& \multirow{2}{*}{Max Reward}
& \textbf{C} & 7.2 / 7.2 & 21.9 / 39.5 & 0.0 / 0.0 & 0.0 / 0.0 & 0.0 / 0.0 & 0.6 / \textbf{0.0} & 0.0 / 0.0 & 0.0 / 0.0 \\
& & \textbf{R} & 0.7 / 0.7 & 0.1 / 0.0 & 381.1 / 381.1 & 380.3 / 380.3 & -2473.9 / \textbf{-2273.7} & -2629.3 / \textbf{-2159.0} & -0.5 / -1.7 & -0.1 / \textbf{0.0} \\
\cmidrule(lr){2-11}
& \multirow{2}{*}{Min Reward}
& \textbf{C} & 0.0 / 20.9 & 4.9 / 9.3 & 9.9 / \textbf{0.0} & 0.0 / 0.0 & 0.1 / \textbf{0.0} & 0.0 / 0.0 & 6.6 / \textbf{1.7} & 0.0 / 0.0 \\
& & \textbf{R} & 0.6 / 0.1 & -9.8 / \textbf{-0.1} & 34.7 / \textbf{36.3} & -20.5 / \textbf{9.0} & -2001.5 / \textbf{-1371.7} & -2804.8 / -2888.8 & -0.1 / \textbf{0.2} & -0.3 / \textbf{0.1} \\
\midrule

\multirow{6}{*}{BEAR-Lag}
& \multirow{2}{*}{Max Cost}
& \textbf{C} & 0.0 / 0.0 & 8.4 / 10.6 & 20.4 / \textbf{8.1} & 49.0 / \textbf{1.0} & 0.0 / 0.0 & 0.0 / 0.0 & 2.4 / 2.4 & 0.0 / 0.0 \\
& & \textbf{R} & -0.2 / \textbf{0.0} & 3.0 / \textbf{9.9} & 324.1 / 6.2 & 411.9 / 17.1 & 850.4 / \textbf{862.7} & 478.4 / \textbf{517.0} & 1.0 / 1.0 & 0.2 / -0.1 \\
\cmidrule(lr){2-11}
& \multirow{2}{*}{Max Reward}
& \textbf{C} & 9.6 / \textbf{7.4} & 5.0 / \textbf{3.4} & 15.4 / 15.4 & 13.0 / \textbf{12.1} & 0.4 / 0.7 & 0.0 / 0.0 & 2.7 / 5.8 & 9.5 / \textbf{0.0} \\
& & \textbf{R} & -0.5 / \textbf{0.5} & -2.6 / \textbf{-2.0} & 256.8 / 256.8 & 305.1 / -44.2 & 663.2 / 537.0 & 621.6 / \textbf{702.0} & -0.2 / -0.8 & 0.3 / -0.2 \\
\cmidrule(lr){2-11}
& \multirow{2}{*}{Min Reward}
& \textbf{C} & 0.0 / 8.8 & 8.5 / \textbf{0.0} & 27.9 / \textbf{2.3} & 86.9 / \textbf{3.9} & 0.0 / 0.0 & 3.6 / \textbf{0.0} & 3.6 / 5.0 & 3.0 / \textbf{2.0} \\
& & \textbf{R} & -0.1 / -1.4 & 6.8 / 0.1 & -45.2 / \textbf{12.0} & 328.3 / 12.9 & 551.2 / -448.1 & 150.7 / -308.3 & -0.4 / \textbf{-0.3} & 0.5 / 0.3 \\
\midrule

\multirow{6}{*}{COptiDICE}
& \multirow{2}{*}{Max Cost}
& \textbf{C} & 6.9 / 11.6 & 5.8 / 7.6 & 0.0 / 12.1 & 15.1 / \textbf{12.5} & 0.0 / 8.9 & 0.0 / 1.2 & 9.8 / \textbf{8.3} & 9.4 / \textbf{7.4} \\
& & \textbf{R} & 0.6 / \textbf{5.3} & -0.3 / \textbf{9.9} & -25.3 / \textbf{124.8} & 95.7 / \textbf{116.0} & -485.4 / \textbf{2647.4} & -1554.3 / \textbf{2593.9} & 1.3 / 0.5 & 0.8 / 0.5 \\
\cmidrule(lr){2-11}
& \multirow{2}{*}{Max Reward}
& \textbf{C} & 18.1 / \textbf{8.2} & 28.2 / \textbf{9.3} & 13.8 / \textbf{7.7} & 12.8 / \textbf{11.3} & 1.2 / \textbf{0.8} & 0.0 / 0.7 & 3.0 / 7.3 & 14.0 / \textbf{9.8} \\
& & \textbf{R} & 0.0 / \textbf{9.6} & 0.2 / \textbf{6.5} & 94.8 / \textbf{152.1} & 94.0 / \textbf{137.5} & 2562.4 / 2530.7 & 181.3 / \textbf{2555.1} & 1.6 / 0.4 & 1.1 / \textbf{1.9} \\
\cmidrule(lr){2-11}
& \multirow{2}{*}{Min Reward}
& \textbf{C} & 20.6 / \textbf{9.1} & 26.5 / \textbf{6.6} & 13.8 / \textbf{9.6} & 13.8 / 16.8 & 0.0 / 1.7 & 0.0 / 4.2 & 8.3 / 9.7 & 6.8 / \textbf{4.8} \\
& & \textbf{R} & 7.8 / \textbf{8.7} & 4.2 / \textbf{9.0} & 94.8 / \textbf{130.5} & 94.8 / \textbf{159.8} & -2080.5 / \textbf{2586.7} & -1875.6 / \textbf{2547.4} & 0.3 / \textbf{1.9} & 1.3 / 1.2 \\

\bottomrule
\end{tabular}%
}
\caption{Cost and reward comparison between non-cost-aware Trajdeleter and cost-aware Trajdeleter. Each entry reports non-cost-aware / cost-aware. Bold cost-aware values indicate improvement over non-cost-aware Trajdeleter: lower cost for \textbf{C}, higher reward for \textbf{R}.}
\label{tab:trajdeleter_cost_aware_compare}
\end{table*}

\begin{table*}[t]
\centering
\tiny
\setlength{\tabcolsep}{3pt}
\renewcommand{\arraystretch}{1.08}
\resizebox{\textwidth}{!}{
\begin{tabular}{llcccccccc}
\toprule
\multirow{2}{*}{\textbf{Policy}}
& \multirow{2}{*}{\textbf{Variant}}
& \multicolumn{2}{c}{\textbf{PointGoal}}
& \multicolumn{2}{c}{\textbf{CarCircle}}
& \multicolumn{2}{c}{\textbf{AntVelocity}}
& \multicolumn{2}{c}{\textbf{PointPush}} \\
\cmidrule(lr){3-4}\cmidrule(lr){5-6}\cmidrule(lr){7-8}\cmidrule(lr){9-10}
& & $C\downarrow$ & $R\uparrow$
& $C\downarrow$ & $R\uparrow$
& $C\downarrow$ & $R\uparrow$
& $C\downarrow$ & $R\uparrow$ \\
\midrule

\multirow{2}{*}{BCQ-Lag}
& with critic/value unlearning& \textbf{17.5} & 18.3
& \textbf{7.6} & \textbf{296.0}
& 21.4 & \textbf{2962.0}
& 14.4 & 2.0 \\
& without critic/value unlearning
& 21.1 & \textbf{21.0}
& 10.4 & 264.1
& \textbf{20.3} & 2954.0
& \textbf{10.3} & \textbf{3.4} \\

\midrule
\multirow{2}{*}{CPQ}
& with critic/value unlearning
& \textbf{4.7} & \textbf{12.6}
& 0.0 & 381.3
& 0.4 & 932.9
& \textbf{1.4} & \textbf{1.9} \\
& without critic/value unlearning
& 6.6 & 8.1
& 0.0 & 381.3
& \textbf{0.0} & \textbf{941.0}
& 8.2 & 0.5 \\

\midrule
\multirow{2}{*}{BEAR-Lag}
& with critic/value unlearning
& \textbf{8.9} & 9.5
& 8.8 & \textbf{244.9}
& 5.8 & \textbf{2517.8}
& 7.4 & \textbf{2.2} \\
& without critic/value unlearning
& 9.5 & \textbf{12.4}
& \textbf{7.8} & 123.1
& \textbf{3.8} & 1149.9
& \textbf{3.9} & 1.8 \\

\midrule
\multirow{2}{*}{COptiDICE}
& with critic/value unlearning
& 17.1 & \textbf{12.7}
& \textbf{38.3} & \textbf{175.0}
& 7.3 & \textbf{2416.0}
& \textbf{5.8} & 1.2 \\
& without critic/value unlearning
& \textbf{11.4} & 8.1
& {48.0} & 122.8
& \textbf{1.5} & 2106.0
& 12.9 & \textbf{2.2} \\

\bottomrule
\end{tabular}
}
\caption{Effect of critic and cost-critic unlearning under Max Cost attack with 5\% poison. Each entry reports the after-unlearning cost and reward. Lower cost and higher reward are better. Bold values indicate the better variant between with and without critic unlearning; ties are left unbolded.}
\vspace{-0.15in}
\label{tab:critic_unlearn_ablation_5poison_after}
\end{table*}
\textbf{Safe-RULE without critic unlearning.} To evaluate the role of critic unlearning, we compare Safe-RULE with a variant that only updates the actor while keeping the corrupted reward and cost critics fixed. As shown in Table~\ref{tab:critic_unlearn_ablation_5poison_after}, critic unlearning improves safety recovery in many settings because it directly removes the poisoned value structure that would otherwise continue to guide the actor toward unsafe behavior.

\begin{table*}[t]
\centering
\tiny
\setlength{\tabcolsep}{5pt}
\renewcommand{\arraystretch}{1.15}
\begin{tabular}{@{}llccc@{}}
\toprule
\multirow{2}{*}{\textbf{Policy}}
& \multirow{2}{*}{\textbf{Task}}
& \multicolumn{3}{c}{\textbf{Setting}} \\
\cmidrule(lr){3-5}
& & \textbf{0\%} & \textbf{50\%} & \textbf{100\%} \\
\midrule

\multirow{4}{*}{BCQ-Lag}
& PointGoal    & 17.5 / 15.6   & 25.1 / 18.8   & 21.3 / 13.7 \\
& CarCircle    & 9.1 / 270.8   & 3.9 / 242.9   & \textbf{9.5 / 271.4} \\
& AntVelocity  & 22.8 / 2959.5 & 22.2 / 2967.5 & 36.2 / 2868.5 \\
& PointPush    & 9.0 / 3.4     & \textbf{6.4 / 3.5} & 5.8 / 2.4 \\
\midrule

\multirow{4}{*}{CPQ}
& PointGoal    & 7.0 / 9.4     & \textbf{4.7 / 12.6} & 21.9 / 0.1 \\
& CarCircle    & \textbf{0.0 / 381.3} & \textbf{0.0 / 381.3} & \textbf{0.0 / 381.3} \\
& AntVelocity  & 0.0 / 950.4   & \textbf{0.9 / 1060.5} & 0.0 / -678.0 \\
& PointPush    & 9.3 / 1.5     & \textbf{1.4 / 1.9} & 0.0 / 0.1 \\
\midrule

\multirow{4}{*}{BEAR-Lag}
& PointGoal    & 5.7 / 6.9     & \textbf{8.9 / 9.5} & 21.9 / 0.1 \\
& CarCircle    & 6.3 / 206.7   & \textbf{8.8 / 244.9} & 1.1 / 0.1 \\
& AntVelocity  & 6.6 / 2477.0  & \textbf{5.9 / 2524.7} & 0.0 / 849.9 \\
& PointPush    & 4.2 / 1.9     & \textbf{7.4 / 2.2} & 2.4 / 1.0 \\
\midrule

\multirow{4}{*}{COptiDICE}
& PointGoal    & \textbf{4.6 / -2.8} & 17.1 / 12.7 & 46.8 / 15.5 \\
& CarCircle    & 30.0 / 113.0  & 36.7 / 177.6  & 50.6 / 48.6 \\
& AntVelocity  & 1.7 / 2586.7  & 6.2 / 2426.7  & \textbf{3.1 / 2807.3} \\
& PointPush    & 11.9 / 0.7    & \textbf{5.8 / 1.2} & 11.4 / 0.7 \\

\bottomrule
\end{tabular}

\caption{
Ablation results under different Reward reference $\bar{Q}^r$ settings. Each entry reports evaluation cost / reward. Bold entries indicate the setting with cost below 10 and the highest reward among safe settings for each policy-task pair. Ties are bolded.
}
\label{tab:percentage_ablation}
\end{table*}

\textbf{Effect of reward reference $\bar{Q}^r$.} Table~\ref{tab:percentage_ablation} studies the effect of different reward reference settings on the unlearning performance. Each entry reports evaluation cost / reward, where lower cost and higher reward are better. We bold the setting that achieves cost below the safety threshold of 10 while obtaining the highest reward among the safe settings for each policy-task pair. The results show that the best percentage is task and algorithm dependent, and in most cases the 50\% yields the best unlearn performance.

\section{Detailed Implementation Settings}
\label{sec:appendix_impl}

\begin{table}[h]
\centering
\small
\renewcommand{\arraystretch}{1.18}
\setlength{\tabcolsep}{6pt}

\begin{tabularx}{\linewidth}{|>{\raggedright\arraybackslash}X|>{\centering\arraybackslash}X|}
\hline
\textbf{Hyperparameter} & \textbf{Value} \\
\hline

\multicolumn{2}{|c|}{\textbf{Training and attack setting}} \\
\hline
Offline safe RL training steps & $10^6$ \\
Discount factor $\gamma$ & $0.99$ \\
Poisoning ratio $\rho$ & $\{0.05, 0.15\}$ \\
Adversarial policy training steps & $10^6$ \\
Unlearning steps & $5000$ \\
\hline

\multicolumn{2}{|c|}{\textbf{Safe-RULE setting}} \\
\hline
Keep critic weight $\alpha_k$ & $1.0$ \\
Forget critic weight $\alpha_f$ & $1.0$ \\
Cost threshold $\kappa$ & 10.0 \\
Forget-cost target & $\kappa+\sigma$ \\
Forget-cost margin $\sigma$ & $0.5$ \\
Reward reference $\bar{Q}^r$ & $50$th percentile of $y_r$ on $\mathcal{D}_k$ \\
\hline

\multicolumn{2}{|c|}{\textbf{Adaptive forget weight}} \\
\hline
Initial $\beta_f$ & $0.1$ \\
$\beta_f$ range & $[0.01, 1.0]$ \\
Update step size $\tau_\beta$ & $0.25$ \\
\hline

\end{tabularx}

\caption{
Hyperparameter settings used in the paper. We only list parameters that are directly used in the Safe-RULE objective, adaptive forget-weight update, attack setup, or experimental training setup.
}
\label{tab:hyperparameters}
\end{table}

\subsection{Hyperparameter Settings}
In Table~\ref{tab:hyperparameters}, we report the hyperparameter settings in our experiments. All the offline RL policies are trained with $10^6$ steps with the discount factor $0.99$, with the same other setting follows the OSRL~\cite{liu2023datasets}. For the Safe-RULE unlearning process, we assume the unlearning steps to 5000 and use the hyperparameter settings in Table~\ref{tab:hyperparameters}.

\textbf{Experiment Hardware Usage.}
All experiments were run on a server with one NVIDIA A100 80GB PCIe GPU, 28 AMD EPYC 7763 CPU cores, and 118 GiB of RAM.

\subsection{Algorithms}
\label{sec:appendix_algorithms}
\textbf{Attack implementation.}
In this paper, we implement three poisoning attacks to evaluate the unlearning defense. The three poisoning attacks are implemented as described in Algorithm~\ref{alg:unified_poison_attack}. We train each adversarial policy for $10^6$ steps and define its reward function based on the attack goal. For Max Cost and Max Reward, we set the poisoned cost to zero and assign a high reward label. For Min Reward, we set the poisoned cost to zero and assign a low reward label, so that the learned policy is pushed toward low-return behavior without directly maximizing cost.

After training the adversarial policy, we roll it out to generate adversarial trajectories and compute the attack score for each trajectory. We then select poisoned samples with the highest attack score, where the score corresponds to high cost, high reward, or low reward depending on the attack type. Finally, we modify the reward and cost labels of the selected samples and inject them into the original dataset $\mathcal{D}$ to obtain the poisoned offline dataset $\widetilde{\mathcal{D}}$.

\begin{algorithm2e}[h]
\caption{Offline Dataset Poisoning Attack}
\label{alg:unified_poison_attack}
\DontPrintSemicolon
\SetKwInOut{Input}{Input}
\SetKwInOut{Output}{Output}

\Input{Clean offline dataset $\mathcal{D}$; environment $\mathcal{M}$; attack type 
$z \in \{\mathrm{MC}, \mathrm{MR}, \mathrm{MinR}\}$; poisoning ratio $\rho$; 
adversarial policy optimization steps $T$; number of rollout trajectories $N$.}
\Output{Poisoned offline dataset $\widetilde{\mathcal{D}}$.}

Initialize adversarial policy $\pi_{\mathrm{adv}}$\;

Define the adversarial reward according to the attack type:
\[
R_z(s,a,r,c)=
\begin{cases}
c, & z=\mathrm{MC},\\
r, & z=\mathrm{MR},\\
-r, & z=\mathrm{MinR}.
\end{cases}
\]

\For{$t=1,\ldots,T$}{
    Roll out $\pi_{\mathrm{adv}}$ in $\mathcal{M}$\;
    Update $\pi_{\mathrm{adv}}$ to maximize 
    $\mathbb{E}_{\pi_{\mathrm{adv}}}\left[\sum_{t}\gamma^t R_z(s_t,a_t,r_t,c_t)\right]$\;
}

Roll out $\pi_{\mathrm{adv}}$ for $N$ trajectories and obtain 
$\mathcal{D}_{\mathrm{adv}}=\{\tau_i\}_{i=1}^{N}$\;

Compute the attack score for each trajectory:
\[
S_z(\tau_i)=\sum_{t}\gamma^t R_z(s_t,a_t,r_t,c_t).
\]

Select the top $\lceil \rho |\mathcal{D}| \rceil$ poisoned samples from 
$\mathcal{D}_{\mathrm{adv}}$ according to $S_z$\;

Set $r_{\mathrm{high}}$ as a high reward value from $\mathcal{D}$ and 
$r_{\mathrm{low}}$ as a low reward value from $\mathcal{D}$\;

\For{each selected transition $(s,a,r,c,s')$}{
    \eIf{$z \in \{\mathrm{MC}, \mathrm{MR}\}$}{
        Set poisoned label $\tilde{c} \gets 0$ and $\tilde{r} \gets r_{\mathrm{high}}$\;
    }{
        Set poisoned label $\tilde{c} \gets 0$ and $\tilde{r} \gets r_{\mathrm{low}}$\;
    }
    Add $(s,a,\tilde{r},\tilde{c},s')$ to $\mathcal{D}_f$\;
}

Construct the poisoned dataset:
\[
\widetilde{\mathcal{D}} = \mathcal{D} \cup \mathcal{D}_f .
\]

\Return $\widetilde{\mathcal{D}}$\;
\end{algorithm2e}

\textbf{Safe-RULE implementation.} Our Safe-RULE algorithm is described in Algorithm~\ref{alg:safe_rule}. The algorithm alternates between critic unlearning and actor unlearning. The critic update preserves value estimation on the keep set and suppresses the value structure on the forget set. The actor update then combines the keep objective and forget objective using an adaptive forget weight.

For COptiDICE\cite{lee2022coptidice}, the implementation is slightly different because it does not use explicit reward and cost $Q$-critics. Instead, COptiDICE learns stationary-distribution correction through the $\nu$-network and the $\chi$-network. Therefore, when applying Safe-RULE to COptiDICE, we use the $\nu$-network as the reward-side value proxy and the $\chi$-network as the cost-side value proxy. During unlearning, we first update the $\nu$-network, $\chi$-network, and the Lagrange multiplier using only the keep data $\mathcal{D}_k$, following the original COptiDICE objective, so that the density-ratio estimation and cost constraint remain stable on the clean distribution. Then, we freeze these value networks and update the actor. On the keep data, the actor is trained with the standard COptiDICE weighted behavior-cloning objective, where the weight is computed from the $\nu$-network. On the forget data, we compute the forget-state weights using both $\nu$ and $\chi$: the $\nu$-based weight suppresses the reward-side preference of the poisoned actions, while the normalized $\chi$ value increases the penalty for actions that should be treated as costly. In practice, this is implemented by reducing the log-probability of the forget actions under the actor distribution, weighted by the reward-side and cost-side proxies. The adaptive forget weight is then updated according to the gap between the target cost level $\kappa+\sigma$ and the current $\chi$ value on $\mathcal{D}_f$. This allows Safe-RULE to apply the same keep-forget principle to COptiDICE, even though it does not have the standard actor-critic structure used by BCQ-Lag, BEAR-Lag, and CPQ.

\begin{algorithm2e}[h]
\caption{Safe-RULE: Safe Reinforcement Unlearning}
\label{alg:safe_rule}
\DontPrintSemicolon
\SetKwInOut{Input}{Input}
\SetKwInOut{Output}{Output}

\Input{Corrupted policy $\pi_\theta$; reward critic $Q^r_\phi$; cost critic $Q^c_\psi$; keep dataset $\mathcal{D}_k$; forget dataset $\mathcal{D}_f$; safety threshold $\kappa$; margin $\sigma$; unlearning steps $T$; critic weights $\alpha_k,\alpha_f$; adaptive forget weight $\beta_f$.}
\Output{Unlearned safe policy $\pi_\theta$.}

Initialize $\beta_f \gets 0.1$\;

\For{$t=1,\ldots,T$}{
    Sample mini-batch $\mathcal{B}_k$ from $\mathcal{D}_k$\;
    Sample mini-batch $\mathcal{B}_f$ from $\mathcal{D}_f$\;

    \tcp{Critic unlearning}
    Compute keep Bellman targets $y_r$ and $y_c$ on $\mathcal{B}_k$ using the underlying offline safe RL algorithm\;

    Compute critic keep losses:
    \[
    \mathcal{L}_{\mathrm{critic}}^r(\mathcal{B}_k)
    =
    \mathbb{E}_{(s,a)\sim\mathcal{B}_k}
    \left[
    (Q^r_\phi(s,a)-y_r)^2
    \right],
    \]
    \[
    \mathcal{L}_{\mathrm{critic}}^c(\mathcal{B}_k)
    =
    \mathbb{E}_{(s,a)\sim\mathcal{B}_k}
    \left[
    (Q^c_\psi(s,a)-y_c)^2
    \right].
    \]

    Set reward reference $\bar{Q}^r$ as the low quantile of $y_r$ on $\mathcal{B}_k$\;
    Set cost reference $\bar{Q}^c \gets \kappa+\sigma$\;

    Compute critic forget losses:
    \[
    \mathcal{L}_{\mathrm{critic}}^r(\mathcal{B}_f)
    =
    \mathbb{E}_{(s,a)\sim\mathcal{B}_f}
    \left[
    \mathrm{softplus}(Q^r_\phi(s,a)-\bar{Q}^r)
    \right],
    \]
    \[
    \mathcal{L}_{\mathrm{critic}}^c(\mathcal{B}_f)
    =
    \mathbb{E}_{(s,a)\sim\mathcal{B}_f}
    \left[
    \mathrm{softplus}(\kappa+\sigma-Q^c_\psi(s,a))
    \right].
    \]

    Update $Q^r_\phi$ using:
    \[
    \mathcal{L}_{\mathrm{critic}}^r
    =
    \alpha_k\mathcal{L}_{\mathrm{critic}}^r(\mathcal{B}_k)
    +
    \alpha_f\mathcal{L}_{\mathrm{critic}}^r(\mathcal{B}_f).
    \]

    Update $Q^c_\psi$ using:
    \[
    \mathcal{L}_{\mathrm{critic}}^c
    =
    \alpha_k\mathcal{L}_{\mathrm{critic}}^c(\mathcal{B}_k)
    +
    \alpha_f\mathcal{L}_{\mathrm{critic}}^c(\mathcal{B}_f).
    \]

    \tcp{Actor unlearning}
    Compute keep actor loss:
    \[
    \mathcal{L}_{\mathrm{actor}}(\mathcal{B}_k)
    =
    -\mathbb{E}_{\mathcal{B}_k}
    \left[
    \mathbb{I}[Q^c_\psi(s,\pi_\theta(s)) \le \kappa]Q^r_\phi(s,\pi_\theta(s))
    \right]
    +
    \mathbb{E}_{\mathcal{B}_k}
    \left[
    \mathrm{softplus}(Q^c_\psi(s,\pi_\theta(s))-\kappa)
    \right].
    \]

    Compute forget actor loss:
    \[
    \mathcal{L}_{\mathrm{actor}}(\mathcal{D}_f)
    =
    \mathbb{E}_{\mathcal{D}_f}
    \left[
    \mathbb{I}[Q^c_\psi(s,\pi_\theta(s)) \ge \kappa]Q^r_\phi(s,\pi_\theta(s))
    \right]
    +
    \mathbb{E}_{\mathcal{B}_f}
    \left[
    \mathrm{softplus}(\kappa+\sigma-Q^c_\psi(s,\pi_\theta(s)))
    \right].
    \]

    Compute forgetting gap:
    \[
    \Delta
    =
    \kappa+\sigma
    -
    \mathbb{E}_{\mathcal{B}_f}
    \left[
    Q^c_\psi(s,\pi_\theta(s))
    \right].
    \]

    Update adaptive forget weight:
    \[
    \beta_f
    \gets
    \mathrm{clip}
    \left(
    \beta_f+\Delta\tau_\beta,\,
    \beta_{\min},\,
    \beta_{\max}
    \right).
    \]

    Update actor $\pi_\theta$ using:
    \[
    \mathcal{L}_{\mathrm{actor}}^{\mathrm{total}}
    =
    \mathcal{L}_{\mathrm{actor}}(\mathcal{B}_k)
    +
    \beta_f
    \mathcal{L}_{\mathrm{actor}}(\mathcal{B}_f).
    \]

    Soft-update target networks if used by the backbone algorithm\;
}

\Return $\pi_\theta$\;
\end{algorithm2e}